\documentclass[acmsmall,authorversion]{acmart}

\usepackage{booktabs}
\usepackage{multirow}
\usepackage{nicefrac}
\usepackage{siunitx}
\usepackage{enumitem}
\usepackage{subcaption}
\usepackage{tikz}
\usetikzlibrary{positioning}
\usepackage{graphicx}
\usepackage{algorithm}
\usepackage{algorithmic}
\usepackage{tcolorbox}
\usepackage{mdframed}
\usepackage{cleveref}
\usepackage{amsmath}

\newtheorem{theorem}{Theorem}

\newtheorem{proposition}[theorem]{Proposition}

\theoremstyle{definition}
\newtheorem{definition}{Definition}
\newtheorem{remark}{Remark}
\newtheorem{example}{Example}

\newcommand{\Zn}{\mathbb{Z}_n}
\newcommand{\Sn}{\mathfrak{S}_n}
\newcommand{\BF}{\mathrm{BF}}
\newcommand{\Rot}{\mathcal{R}}
\newcommand{\Bij}{\mathcal{B}}
\newcommand{\Aff}{\mathcal{A}}
\newcommand{\Ident}{\mathcal{I}}
\newcommand{\D}{\mathcal{D}}

\renewcommand\footnotetextcopyrightpermission[1]{}
\setcopyright{none}
\acmConference{}{}{}
\acmYear{}
\acmDOI{}
\acmISBN{}

\title{Bayesian Wind Tunnels for Model Selection}
\subtitle{Do Transformers Select Hypothesis Classes Like Bayesian Scientists?}

\author{Siddhartha R. Dalal}
\affiliation{
  \institution{Columbia University}
  \department{School of Professional Studies and Department of Statistics}
  \city{New York}
  \state{NY}
  \country{USA}
}
\email{sd2803@columbia.edu}

\author{Vishal Misra}
\affiliation{
  \institution{Columbia University}
  \department{Department of Computer Science}
  \city{New York}
  \state{NY}
  \country{USA}
}
\email{vishal.misra@columbia.edu}

\author{Abhay Parekh}
\affiliation{
  \institution{Columbia University}
  \city{New York}
  \state{NY}
  \country{USA}
}
\email{akp2185@columbia.edu}
\authorsaddresses{Authors' emails: sd2803@columbia.edu, vishal.misra@columbia.edu, akp2185@columbia.edu. Authors are listed in alphabetical order.}

\begin{document}

\begin{abstract}
Prior work has shown that transformers can perform exact Bayesian filtering \emph{within} a fixed hypothesis class.
Can they also perform Bayesian \emph{model selection}---identifying the correct hypothesis class from data?
We introduce \emph{model-selection Bayesian wind tunnels}: controlled environments where ground-truth posteriors over hypothesis classes are available in closed form.
Using fixed-point-free involutions---whose defining property $f(f(x)) = x$ is purely relational---a 2.8M-parameter transformer achieves 0.01-bit entropy agreement with the Bayesian optimum (3 seeds), with both integer tokens and opaque symbols whose meanings change every episode.
This extends to \emph{non-nested} comparisons: involutions vs.\ 3-cycles (where neither class is a subset of the other) achieve class posterior MAE under 0.001---demonstrating genuine Bayesian model selection beyond simplicity/subset bias.
We then identify a sharp \emph{perceptual access condition}: when the discriminative statistic requires arithmetic---modular addition (rotations) or multiplication ($f(x) = cx \bmod p$)---model selection succeeds with integer tokens but fails completely with opaque symbols, and this boundary persists under $112\times$ scaling (2.8M to 316M parameters).
A stationarity control confirms the operative factor: opaque tokens with a \emph{fixed} relabeling succeed (0.009-bit MAE), proving that stable semantics---not integer identity---enable circuit compilation.
Header subtask diagnostics localize the failure: a model trained on permutation inversion alone achieves 100\% accuracy, whereas a model trained on the composed operation (header inversion + modular arithmetic) reaches only 18.6\% accuracy---about $3\times$ the 6.25\% baseline but far short of the parsing subtask---localizing the bottleneck to the composition step rather than header parsing.
Probing frontier LLMs on the same tasks shows qualitative Bayesian behavior but a large calibration gap ($\sim$$55\times$)---measured through lossy top-5-logprob and low-sample probes against a full-softmax purpose-trained model, and therefore directional rather than exact.
\end{abstract}

\maketitle

\section{Introduction}
\label{sec:intro}

A striking property of in-context learning in large language models~\citep{brown2020language} is the ability to infer \emph{structure} from sparse evidence.
Given a few input-output pairs consistent with a cyclic rotation, frontier models confidently predict the rotation rather than treating the function as an arbitrary bijection.
This behavior---selecting the simplest consistent hypothesis class before filtering within it---is a hallmark of Bayesian model selection~\citep{mackay1992bayesian}, also known as Occam's razor.

Two broad camps have emerged in understanding this phenomenon.
The \emph{simulator} view holds that transformers trained on sequences from a generative process learn to track posteriors within a fixed hypothesis class, effectively implementing Bayesian filtering as a byproduct of next-token prediction~\citep{aggarwal2025bayesian1,xie2022explanation}.
The \emph{scientist} view holds that transformers go further: they select among hypothesis classes from data, performing the equivalent of Bayesian model comparison~\citep{garg2022can,vonoswald2023transformers}.
The simulator view is now well-supported experimentally---Bayesian wind tunnel (BWT) experiments have demonstrated exact Bayesian filtering within bijections, hidden Markov models, and regression tasks~\citep{aggarwal2025bayesian1}---but the scientist view lacks comparably controlled, quantitative evidence.

The challenge is methodological: testing model selection requires environments where (i) the ground-truth posterior over hypothesis \emph{classes} is available in closed form, (ii) the generative process genuinely selects a class before generating data, and (iii) the experimenter can distinguish model selection from within-class filtering.
Without closed-form ground truth, one cannot measure how far a model's implicit class posterior deviates from the Bayesian optimum; without a genuine two-level generative process, apparent model selection may reduce to pattern matching on surface statistics.
The wind tunnel methodology addresses both requirements by construction: it pairs a tractable generative process with exact Bayesian benchmarks, enabling bit-level precision measurements of model selection behavior.

Previous BWTs test \emph{filtering}---computing $P(\theta \mid M, \D_k)$ where the hypothesis class $M$ is known---but not \emph{model selection}---computing $P(M \mid \D_k)$ where the class itself is uncertain.
Model selection is a qualitatively harder Bayesian computation: the model must simultaneously track posteriors over classes and parameters, marginalizing over the latter to update the former.

In this paper, we extend the BWT methodology to model selection.
We construct generative environments where:
\begin{enumerate}[itemsep=2pt]
  \item The data-generating process first samples a hypothesis class $M$ from a prior $\pi(M)$, then samples a function $f$ from that class.
  \item The Bayesian posterior over the class, $P(M \mid \D_k)$, is available in closed form after $k$ observations.
  \item The Bayesian predictive distribution, marginalizing over both class and function uncertainty, is tractable.
\end{enumerate}

This allows us to ask, with the same bit-level precision as the original BWTs: do transformers perform quantitatively correct Bayesian model selection?

\paragraph{Contributions.}
\begin{itemize}[itemsep=2pt]
  \item We derive closed-form posteriors for model selection over function-class hierarchies on $\Zn$, covering both nested and non-nested comparisons (\Cref{sec:setup}).
  \item We design a \emph{label-invariant} model-selection BWT using fixed-point-free involutions, whose defining property $f(f(x)) = x$ is purely relational, and show that a 2.8M-parameter transformer achieves 0.01-bit entropy agreement with the Bayesian optimum (\Cref{sec:variant3,sec:inv-experiments}).
  \item We demonstrate that model selection extends to \emph{non-nested} comparisons---involutions vs.\ 3-cycles, where neither class is a subset of the other---with class posterior MAE under 0.001, confirming that the capability generalizes beyond simplicity/subset bias to genuine discrimination between disjoint hypothesis classes (\Cref{sec:nonnested}).
  \item We identify a \emph{perceptual access condition}: model selection over arithmetic structures (rotations, scalar multiplication) succeeds with integer tokens but fails with opaque symbols, while relational structures (involutions, 3-cycles) succeed with both---sharply delineating the boundary of gradient-compiled inference (\Cref{sec:rotation-experiments}).
  \item A stationarity control (fixed opaque relabeling, 0.009-bit MAE) proves the operative factor is semantic stationarity, not integer identity; header subtask diagnostics localize the failure to the composition of header parsing with arithmetic, not to header reading itself.
  \item We probe frontier LLMs on the same task and find qualitative Bayesian behavior but a $55\times$ calibration gap relative to purpose-training (\Cref{sec:llm-probing}).
\end{itemize}

\section{Related Work}
\label{sec:related}

\paragraph{In-context learning as implicit Bayesian inference.}
\citet{xie2022explanation} provided the first theoretical framework for viewing in-context learning as implicit Bayesian inference, showing that a transformer trained on mixtures of hidden Markov models learns to infer the latent concept from in-context examples.
Their analysis establishes that the pretraining objective implicitly performs posterior inference over a latent variable selecting the data-generating process.
Our work differs in two respects: (i) we test model selection over \emph{nested} hypothesis classes with different structural complexity---not concept identification within a single mixture---and (ii) we provide closed-form Bayesian benchmarks enabling bit-level precision measurement, rather than asymptotic or qualitative comparisons.

\paragraph{Transformers as meta-learners.}
\citet{garg2022can} demonstrated that transformers trained on function classes (linear functions, decision trees) can in-context learn functions competitive with task-specific algorithms.
\citet{vonoswald2023transformers} showed that linear attention layers implement one step of gradient descent on a least-squares objective and that transformers can meta-learn learning algorithms through their forward pass.
These results establish that transformers can \emph{implement} learning algorithms through their forward pass, but do not test whether the resulting inference is quantitatively Bayesian, nor whether the model selects among competing hypothesis classes.
Our model-selection BWTs directly test this: we measure whether the implicit class posterior tracks the Bayesian posterior with bit-level precision.

\paragraph{Grokking and delayed generalization.}
\citet{power2022grokking} discovered that small models trained on algorithmic tasks can exhibit sudden generalization long after memorization, a phenomenon termed grokking, whose mechanism has since been reverse-engineered for modular arithmetic and group operations~\citep{nanda2023progress,chughtai2023universality}.
Our rotation experiments exhibit similar dynamics: 80K steps of plateau followed by sudden emergence of model selection (characterized here from a single representative run; a multi-seed characterization of the transition is left to future work).
The involution experiments, by contrast, show smooth learning from the start.
We interpret this contrast through the lens of the perceptual access condition: grokking occurs when the model must consolidate a global algebraic circuit (modular arithmetic for rotations), while smooth learning occurs when the relevant computation (relational matching for involutions) aligns with the attention mechanism's inductive bias.

\paragraph{Amortized Bayesian inference and Prior-Data Fitted Networks.}
Prior-Data Fitted Networks~\citep{mueller2022pfn,wang2025pfn} and related amortized Bayesian methods~\citep{reuter2025transformers} train transformers on synthetic data from task priors to learn posterior predictive distributions.
Our approach shares the spirit of learning from synthetic Bayesian tasks, but differs in focus: PFNs target practical prediction tasks and evaluate predictive accuracy, while our wind tunnels target the inference \emph{mechanism} itself and evaluate whether the model's internal posteriors match the Bayesian optimum with bit-level precision.
Recent work on hierarchical in-context learning~\citep{panwar2024hierarchical} and exchangeable sequence modeling as Bayesian inference~\citep{bae2024exchangeable} provides complementary theoretical perspectives on how transformers encode posteriors over latent environments; our contribution is an empirical testbed where these posteriors are exactly computable.

\paragraph{Expressivity and circuit complexity of transformers.}
A parallel line of work characterizes what bounded-depth transformers can compute: log-precision transformers lie within $\mathrm{TC}^0$~\citep{merrill2023parallelism}, they tend to learn ``shortcuts'' to automata rather than emulating them step by step~\citep{liu2023shortcuts}, and their attention computations can be expressed in the RASP programming model~\citep{weiss2021rasp}.
Our polynomial barrier (\Cref{sec:polynomial}) is an empirical, learnability-side counterpart to these expressivity results: we observe that gradient descent fails to compile the Vandermonde/interpolation circuit for degree-$d$ detection even when capacity is ample.
We connect the two only informally---establishing a formal depth bound for the discriminative statistic is beyond our scope---but the correspondence suggests why the boundary tracks arithmetic depth.

\paragraph{Bayesian wind tunnels.}
The BWT methodology was introduced by \citet{aggarwal2025bayesian1}, who demonstrated exact Bayesian filtering in bijection, HMM, and regression tasks.
Subsequent work analyzed the gradient dynamics underlying this behavior~\citep{aggarwal2025gradient2} and extended the approach to production-scale LLMs~\citep{aggarwal2025geometric3}.
All prior BWTs test \emph{within-class} inference: the hypothesis class is fixed and the model tracks posteriors over parameters.
We extend the methodology to \emph{between-class} inference, testing whether transformers perform Bayesian model selection---a qualitatively harder computation requiring marginalization over parameters to update class posteriors.

\section{Model-Selection Wind Tunnel: Setup}
\label{sec:setup}

\subsection{Symbol Space and Function Classes}
\label{sec:classes}

Let $\Zn = \{0, 1, \ldots, n-1\}$ denote the cyclic group of integers modulo $n$.
We define a nested hierarchy of bijection classes on $\Zn$:

\begin{definition}[Function-Class Hierarchy]
\label{def:hierarchy}
\begin{align}
  \Ident &= \{ \mathrm{id} \} && |\Ident| = 1 \\
  \Rot &= \{ f_c : x \mapsto x + c \bmod n \mid c \in \Zn \} && |\Rot| = n \\
  \Bij &= \{ \text{all bijections on } \Zn \} && |\Bij| = n!
\end{align}
with the strict nesting $\Ident \subset \Rot \subset \Bij$.
\end{definition}

\begin{remark}
For $n$ prime, we can insert the affine class $\Aff = \{ x \mapsto ax + b \bmod n \mid a \in \Zn^*, b \in \Zn \}$ with $|\Aff| = n(n-1)$ between $\Rot$ and $\Bij$.
We focus on $\Rot$ vs.\ $\Bij$ in this paper for clarity, and treat the full hierarchy as an extension.
\end{remark}

\subsection{Generative Process}
\label{sec:generative}

The data-generating process for each sequence is:
\begin{enumerate}[itemsep=2pt]
  \item \textbf{Class selection.} Draw $M \in \{\Rot, \Bij\}$ from prior $\pi$:
    \begin{equation}
      P(M = \Rot) = \pi, \quad P(M = \Bij) = 1 - \pi.
    \end{equation}
  \item \textbf{Function selection.} Draw $f$ uniformly from $M$:
    \begin{equation}
      f \sim \mathrm{Uniform}(M).
    \end{equation}
  \item \textbf{Observation sequence.} Choose distinct inputs $x_1, \ldots, x_K \in \Zn$ (without replacement) and present the pairs $(x_1, f(x_1)), (x_2, f(x_2)), \ldots$
  \item \textbf{Prediction query.} At position $k+1$, present input $x_{k+1}$ and require the model to produce a distribution over $f(x_{k+1})$.
\end{enumerate}

We denote the observed data after $k$ pairs as $\D_k = \{(x_1, y_1), \ldots, (x_k, y_k)\}$ where $y_i = f(x_i)$.

\subsection{Posterior over Classes}
\label{sec:class-posterior}

\begin{proposition}[Class Posterior]
\label{prop:class-posterior}
Let $\D_k$ be a set of $k$ input-output pairs with distinct inputs and distinct outputs (as required by any bijection).
Define the indicator
\begin{equation}
  \mathbb{1}_\Rot(\D_k) =
  \begin{cases}
    1 & \text{if there exists } c \in \Zn \text{ such that } y_i = x_i + c \bmod n \text{ for all } i, \\
    0 & \text{otherwise.}
  \end{cases}
\end{equation}

Then:
\begin{equation}
\label{eq:class-posterior}
  P(\Rot \mid \D_k) =
  \begin{cases}
    \displaystyle\frac{\pi \cdot \frac{1}{n}}{\pi \cdot \frac{1}{n} + (1-\pi) \cdot \frac{(n-k)!}{n!}} & \text{if } \mathbb{1}_\Rot(\D_k) = 1, \\[12pt]
    0 & \text{if } \mathbb{1}_\Rot(\D_k) = 0.
  \end{cases}
\end{equation}
\end{proposition}

\begin{proof}
We compute the marginal likelihood under each class.

\textbf{Under $\Rot$:} The class contains $n$ functions (one per shift $c$).
The first observation $(x_1, y_1)$ determines $c = y_1 - x_1 \bmod n$ uniquely.
If subsequent observations are consistent with this $c$, all have likelihood 1 under the determined rotation; otherwise the likelihood is 0.
Hence:
\begin{equation}
\label{eq:lik-rot}
  P(\D_k \mid \Rot) = \frac{1}{n} \cdot \mathbb{1}_\Rot(\D_k).
\end{equation}

\textbf{Under $\Bij$:} The class contains $n!$ bijections.
The first observation is consistent with $(n-1)!$ bijections (any that map $x_1 \mapsto y_1$), the second with $(n-2)!$, and so on.
The number consistent with all $k$ observations is $(n-k)!$.
Under a uniform prior over bijections:
\begin{equation}
\label{eq:lik-bij}
  P(\D_k \mid \Bij) = \frac{(n-k)!}{n!} = \frac{1}{n(n-1)\cdots(n-k+1)} = \frac{1}{\binom{n}{k} \cdot k!}.
\end{equation}

Applying Bayes' rule:
\begin{equation}
  P(\Rot \mid \D_k) = \frac{\pi \cdot P(\D_k \mid \Rot)}{\pi \cdot P(\D_k \mid \Rot) + (1 - \pi) \cdot P(\D_k \mid \Bij)}.
\end{equation}
Substituting \eqref{eq:lik-rot} and \eqref{eq:lik-bij} yields \eqref{eq:class-posterior}.
\end{proof}

\begin{definition}[Bayes Factor]
\label{def:bayes-factor}
The Bayes factor in favor of $\Rot$ over $\Bij$ after $k$ observations consistent with a rotation is:
\begin{equation}
\label{eq:bayes-factor}
  \BF(\Rot : \Bij \mid \D_k) = \frac{P(\D_k \mid \Rot)}{P(\D_k \mid \Bij)} = \frac{(n-1)!}{(n-k)!} = \prod_{j=1}^{k-1}(n-j).
\end{equation}
\end{definition}

\begin{remark}[Canonical Occam setting]
\label{rem:occam}
The rotation-vs-bijection setup is a canonical instance of Bayesian model comparison~\citep{mackay1992bayesian,kass1995bayes}: a small, structured class ($|\Rot| = n$) is nested inside a large, unstructured class ($|\Bij| = n!$).
The Occam factor---the ratio of prior predictive likelihoods---automatically penalizes the more complex class.
Intuitively, a uniform prior over $n!$ bijections assigns each function probability $1/n!$, while a uniform prior over $n$ rotations assigns each $1/n$.
When the data is consistent with a rotation, the structured class ``predicted'' this particular function with probability $1/n$ while the unstructured class assigned it only $1/n!$---the likelihood ratio is $(n-1)!/1$, the Occam factor.
This is not a toy artifact: the same mechanism operates whenever a parametric family is compared against a nonparametric alternative.
\end{remark}

\begin{example}[$n = 16$]
\label{ex:n16}
With equal priors ($\pi = 1/2$), the posterior probability of the rotation class after $k$ observations consistent with a rotation is:

\medskip
\begin{center}
\begin{tabular}{cccc}
\toprule
$k$ & $\BF(\Rot:\Bij)$ & $P(\Rot \mid \D_k)$ & bits of evidence \\
\midrule
1 & 1 & 0.500 & 0.0 \\
2 & 15 & 0.938 & 3.9 \\
3 & 210 & 0.995 & 7.7 \\
4 & 2{,}730 & 0.9996 & 11.4 \\
5 & 32{,}760 & 0.99997 & 15.0 \\
\bottomrule
\end{tabular}
\end{center}

\medskip
\noindent
Here ``bits of evidence'' is $\log_2 \BF$.
After just 3 rotation-consistent observations, the Bayesian reasoner has accumulated 7.7 bits of evidence for the rotation hypothesis.
The Occam factor does enormous work: each observation consistent with a rotation is $\sim$15$\times$ more likely under $\Rot$ than $\Bij$, because the rotation class concentrates its probability mass on $n = 16$ functions while the bijection class spreads it over $16! \approx 2 \times 10^{13}$.
This super-exponential growth of the Bayes factor---$(n-1)!/(n-k)!$---is the quantitative signature of Occam's razor penalizing unnecessary complexity.
\end{example}

\subsection{Predictive Distribution}
\label{sec:predictive}

\begin{proposition}[Bayesian Predictive]
\label{prop:predictive}
Given data $\D_k$ consistent with a rotation $c = y_1 - x_1 \bmod n$, the Bayesian predictive distribution for a new input $x_{k+1} \notin \{x_1, \ldots, x_k\}$ is:
\begin{equation}
\label{eq:predictive}
  P(f(x_{k+1}) = y \mid \D_k) =
  \begin{cases}
    P(\Rot \mid \D_k) + \dfrac{P(\Bij \mid \D_k)}{n - k} & \text{if } y = x_{k+1} + c \bmod n \text{ and } y \notin \{y_1, \ldots, y_k\}, \\[10pt]
    \dfrac{P(\Bij \mid \D_k)}{n - k} & \text{if } y \neq x_{k+1} + c \bmod n \text{ and } y \notin \{y_1, \ldots, y_k\}, \\[10pt]
    0 & \text{if } y \in \{y_1, \ldots, y_k\}.
  \end{cases}
\end{equation}
\end{proposition}

\begin{proof}
By the law of total probability:
\begin{equation}
  P(f(x_{k+1}) = y \mid \D_k) = \sum_{M \in \{\Rot, \Bij\}} P(f(x_{k+1}) = y \mid M, \D_k) \cdot P(M \mid \D_k).
\end{equation}

Under $\Rot$, given $\D_k$: the rotation $c$ is determined, so $P(f(x_{k+1}) = y \mid \Rot, \D_k) = \mathbb{1}[y = x_{k+1} + c \bmod n]$.

Under $\Bij$, given $\D_k$: there are $(n-k)!$ consistent bijections, each mapping $x_{k+1}$ to a distinct element of $\Zn \setminus \{y_1, \ldots, y_k\}$ (the $n-k$ unused outputs).
By symmetry, each unused output is equally likely:
\begin{equation}
  P(f(x_{k+1}) = y \mid \Bij, \D_k) =
  \begin{cases}
    \frac{1}{n-k} & \text{if } y \notin \{y_1, \ldots, y_k\}, \\
    0 & \text{if } y \in \{y_1, \ldots, y_k\}.
  \end{cases}
\end{equation}

Note that when $\D_k$ is consistent with rotation $c$, the rotation-predicted output $y^* = x_{k+1} + c \bmod n$ is necessarily in $\Zn \setminus \{y_1, \ldots, y_k\}$ (since $f_c$ is a bijection and $x_{k+1} \notin \{x_1,\ldots,x_k\}$).
Combining the two components yields \eqref{eq:predictive}.
\end{proof}

\subsection{Predictive Entropy}
\label{sec:entropy}

\begin{proposition}[Predictive Entropy under Model Selection]
\label{prop:entropy}
Let $p_\Rot = P(\Rot \mid \D_k)$, $p_\Bij = 1 - p_\Rot$, and $m = n - k$ (the number of unused outputs).
The predictive entropy for $f(x_{k+1})$ given rotation-consistent data is:
\begin{equation}
\label{eq:entropy}
  H(f(x_{k+1}) \mid \D_k) = -\left(p_\Rot + \frac{p_\Bij}{m}\right) \log_2\!\left(p_\Rot + \frac{p_\Bij}{m}\right) - (m-1)\frac{p_\Bij}{m} \log_2 \frac{p_\Bij}{m}.
\end{equation}
\end{proposition}

\begin{proof}
The predictive distribution has two types of atoms: one output $y^*$ with mass $p_\Rot + p_\Bij/m$, and $m - 1$ other unused outputs each with mass $p_\Bij/m$.
The entropy is the negative sum of $p \log_2 p$ over these $m$ atoms.
\end{proof}

\begin{remark}[Entropy trajectory]
\label{rem:entropy-trajectory}
The entropy trajectory under model selection differs qualitatively from the within-class BWTs:
\begin{itemize}[itemsep=2pt]
  \item \textbf{Pure bijection BWT} \citep{aggarwal2025bayesian1}: $H_{\mathrm{Bayes}}(k) = \log_2(n - k)$, a monotone staircase from $\log_2 n$ bits to zero (each step is $\log_2\!\tfrac{n-k}{n-k-1}$, small for early $k$).
  \item \textbf{Model-selection BWT}: entropy drops \emph{much faster} because model selection concentrates mass on the rotation prediction.
  At $k = 3$ with $n = 16$ and equal priors, $H \approx 0.056$ bits, compared to $\log_2(13) \approx 3.7$ bits for pure-bijection filtering.
\end{itemize}
This ${\approx}66\times$ entropy reduction is the quantitative signature of Occam's razor.
\end{remark}

\section{Motivating Experiment: Rotations}
\label{sec:rotation-experiments}

We first test model selection using the rotation-vs-bijection setup from \Cref{sec:setup}.
A 2.8M-parameter transformer (6 layers, 6 heads, $d_{\mathrm{model}} = 192$, $d_{\mathrm{ff}} = 768$) is trained on sequences from the generative process of \Cref{sec:generative} with $n = 16$ and $\pi = 0.5$.
Each training sequence presents $n - 1 = 15$ input-output pairs $[x_1, f(x_1), \textsc{sep}, \ldots]$, with cross-entropy loss at each output position.
Batches of 64 sequences are generated on-the-fly using AdamW ($\mathrm{lr} = 3 \times 10^{-4}$, cosine schedule, 150K steps).

\paragraph{Evaluation.}
At each position $k$, we compare the model's predictive entropy against the Bayesian optimum from \Cref{prop:entropy}, and extract the model's implicit class posterior:
\begin{equation}
\label{eq:implicit-class-posterior}
  \hat{P}(\Rot \mid \D_k) = \frac{p_{\mathrm{model}}(y^*) - \frac{1}{n-k}}{1 - \frac{1}{n-k}}
\end{equation}
where $y^* = x_{k+1} + c \bmod n$ is the rotation-predicted output.
We report mean absolute entropy error (MAE, in bits) and class posterior MAE.
Note that entropy MAE is a distribution-level metric (it depends on the full predictive distribution, not just the top token), providing a stronger test than top-token accuracy alone.

\paragraph{Key finding: the perceptual access condition.}
We define the \emph{perceptual access condition} as follows: a structured hypothesis class $M$ satisfies the perceptual access condition under a given token encoding if the defining property of $M$ can be evaluated from token identities alone, without requiring arithmetic or ordering information beyond what the encoding provides. 

A truly abstract Bayesian reasoner would behave identically under
per-episode relabelings of the tokens. To see that
Bayesian model selection is equivariant under relabeling of the symbol set,
let $\sigma:\Zn\to\Zn$ be any bijection and define the conjugate function
$f^{\sigma}=\sigma\circ f\circ\sigma^{-1}$.
If a dataset $D_k=\{(x_i,y_i)\}_{i=1}^k$ is relabeled to
$D_k^{\sigma}=\{(\sigma(x_i),\sigma(y_i))\}_{i=1}^k$,
then the Bayesian posterior transforms equivariantly:
\[
P(\mathcal{M}\mid D_k)=P(\mathcal{M}^{\sigma}\mid D_k^{\sigma}),
\]
where $\mathcal{M}^{\sigma}=\{\, f^{\sigma} : f \in \mathcal{M} \,\}$ is the relabeled hypothesis class.

Failures under opaque relabeling therefore diagnose representational
limitations of the learned inference procedure rather than a property
of Bayesian inference itself.

We tested two token encodings: (i) \emph{opaque header}, where symbols are randomly relabeled per sequence with a header encoding the cyclic ordering (as in~\citep{aggarwal2025bayesian1}), and (ii) \emph{integer tokens}, where the tokens $0, 1, \ldots, 15$ are used directly.
With opaque headers, the model completely failed to learn model selection (MAE $= 1.28$ bits, class posterior MAE $= 0.92$).
With integer tokens, model selection succeeded: MAE $= 0.12$ bits, class posterior MAE $= 0.04$.
The model's entropy converged to within 0.02 bits of the Bayesian optimum by $k = 5$, and its implicit class posterior tracked the Bayesian posterior across all positions.

The failure is not due to model capacity or training---the header model trains for the same number of steps and successfully learns bijection elimination (the $\log_2(n-k)$ staircase).
The bottleneck is perceptual: detecting that $y = x + c \bmod n$ requires arithmetic over token identities, which gradient descent does not learn to perform when tokens are arbitrarily relabeled per episode.

\paragraph{Scaling stress test.}
To rule out insufficient capacity, we scaled the opaque rotation experiment from 2.8M to 316M parameters---a $112\times$ increase---under identical data, optimization, and curriculum (\Cref{fig:scaling-wall}).
With integer tokens, scaling improves precision: the 25M model achieves 0.008-bit entropy MAE and the 152M model achieves 0.005 bits.
With opaque tokens, scaling changes nothing: the 25M, 152M, and 316M models converge to 1.20, 1.19, and 1.17 bits respectively---indistinguishable from the 2.8M model's 1.28 bits.
A cross-evaluation further confirms the diagnosis: the integer-trained 25M model (0.008-bit MAE on integers) fails completely on opaque tokens (1.42-bit MAE), despite having both the arithmetic circuit in its weights and the permutation information in the context.
The failure is specifically one of dynamic semantic rebinding---connecting a per-episode permutation to pre-compiled arithmetic---not of missing computation.

\paragraph{Extended training.}
To rule out late grokking, we trained the 25M-parameter model for $500{,}000$ steps ($3.3\times$ the standard run) under both conditions.
The opaque model's entropy MAE remains flat at $1.28$ bits from step $10{,}000$ onward, showing no sign of improvement through the entire extended run.
The integer model continues improving: MAE decreases from $0.167$ at step $200{,}000$ to $0.097$ at step $436{,}000$, demonstrating that additional training helps when the perceptual access condition is met but has no effect when it is not.
The boundary is not a matter of insufficient training time.

\paragraph{Attention diagnostics.}
Analyzing attention patterns across all layers and heads of the 25M-parameter opaque model over 500 episodes (\Cref{fig:attention-heatmap}) reveals that body tokens' mean attention to header positions (0.358) is indistinguishable from the uniform-attention baseline (0.360).
Early layers show slightly elevated header attention (layer 0, head 1: 0.78), but deeper layers---where task-relevant computation would occur---revert to baseline.
The model does not actively ignore the header; it extracts no actionable information from it.
A header ablation confirms this causally: randomizing the header (providing the wrong permutation) or removing it entirely leaves performance unchanged (MAE $1.55$ bits in all three conditions), establishing that gradient descent never compiled a header-reading circuit for the model-selection task.

The failure is not one of header reading per se.
A transformer trained on a permutation-inversion subtask---given the header and an opaque token, predict the corresponding integer---achieves $100\%$ accuracy within $1{,}500$ gradient steps, demonstrating that gradient descent can compile header parsing.
However, a separate model trained on the composed operation---given the header and an opaque input $x$, predict the opaque token for $(x + c) \bmod n$---reaches only $18.6\%$ accuracy after $150{,}000$ steps, about $3\times$ the $6.25\%$ baseline but far short of the $100\%$ achieved on parsing alone.
The bottleneck is thus not reading the permutation but routing it into arithmetic: the circuit-assembly step that would connect header inversion to modular addition is where gradient descent falls well short, even though it compiles each component in isolation.

\begin{figure}[t]
\centering
\includegraphics[width=\columnwidth]{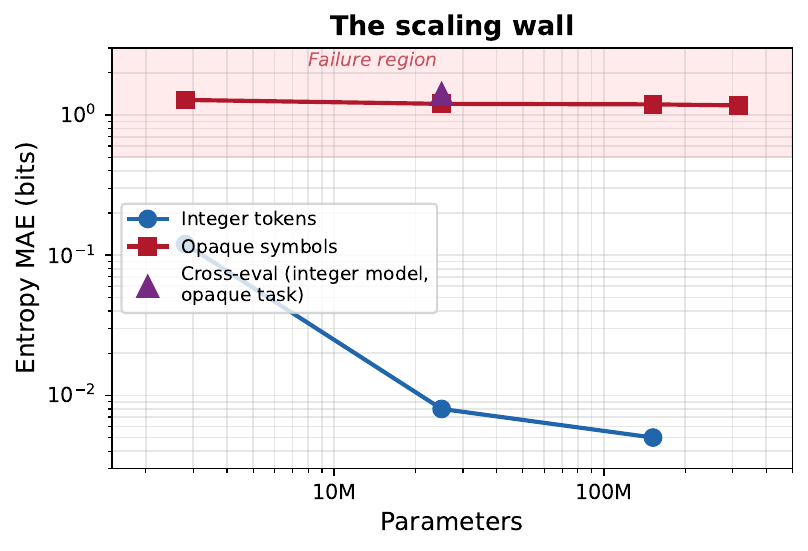}
\caption{The scaling boundary. Entropy MAE vs.\ parameter count for integer tokens (blue) and opaque symbols (red) on the rotation model-selection task.
Integer precision improves with scale ($0.12 \to 0.005$ bits); opaque error is scale-invariant ($1.28 \to 1.17$ bits).
The purple triangle shows cross-evaluation: an integer-trained model tested on opaque tokens fails ($1.42$ bits) despite possessing the arithmetic circuit.}
\label{fig:scaling-wall}
\end{figure}

\begin{figure}[t]
\centering
\includegraphics[width=\columnwidth]{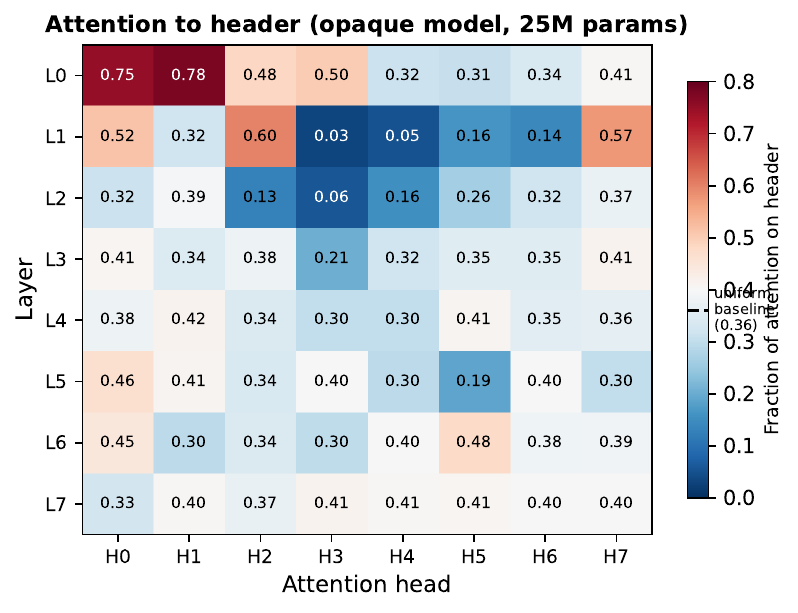}
\caption{Attention to header positions in the 25M-parameter opaque-trained model, averaged over 500 episodes.
Each cell shows the fraction of attention mass that body tokens place on header positions.
The dashed line on the colorbar marks the uniform-attention baseline (0.36).
Layers 2--7 are indistinguishable from uniform; the model cannot extract useful information from the per-episode permutation.}
\label{fig:attention-heatmap}
\end{figure}

\paragraph{Varying group size.}
To confirm that the failure is not specific to $n = 16$, we trained opaque rotation models for $n \in \{8, 32, 64\}$ (\Cref{fig:vary-n}).
All group sizes fail: entropy MAE is $0.63 \pm 0.01$ bits ($n = 8$), $1.84 \pm 0.04$ bits ($n = 32$), and $2.30 \pm 0.02$ bits ($n = 64$), with class-posterior MAE exceeding $0.69$ in every case.
The MAE increases with $n$ (as expected, since the task grows harder with group size), but remains far below the uniform baseline of $\log_2 n$ bits, indicating that the model learns basic bijection elimination but never engages model selection.
The boundary persists across group sizes.

\begin{figure}[t]
\centering
\includegraphics[width=\columnwidth]{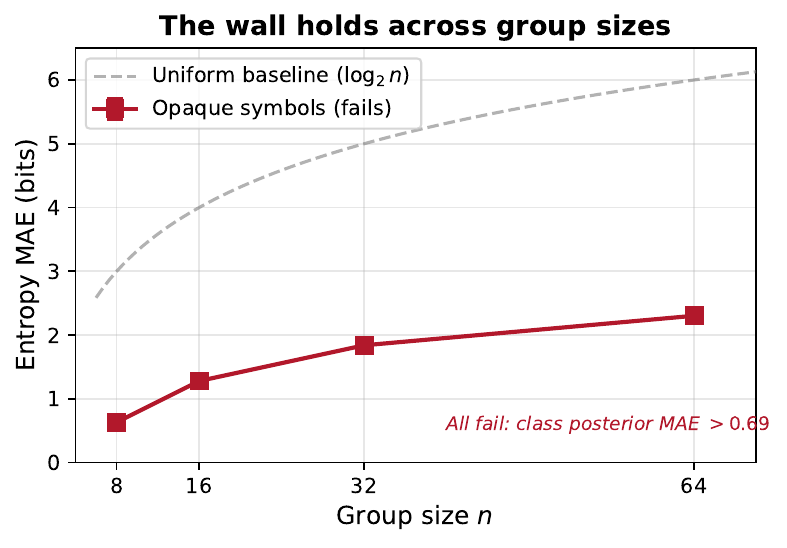}
\caption{Entropy MAE vs.\ group size $n$ for opaque rotation model selection (3 seeds, mean $\pm$ std).
The boundary persists across all group sizes tested.
MAE rises with $n$ but remains below the uniform baseline ($\log_2 n$, dashed), indicating the model learns bijection elimination but not model selection.}
\label{fig:vary-n}
\end{figure}

\paragraph{Stationarity control.}
To isolate stationarity as the operative factor---rather than integer identity per se---we trained a model with opaque tokens under a \emph{fixed} permutation held constant across all episodes.
The tokens are arbitrary labels (not integers), but the same label always denotes the same value.
This model achieves $0.009$-bit entropy MAE and class-posterior MAE of $0.002$ (3 seeds, $0.009 \pm 0.0005$ bits), matching the involution result and surpassing the integer condition ($0.12$ bits) by an order of magnitude.
The operative factor is stationarity of the token-to-meaning mapping: with any stable encoding, gradient descent discovers modular arithmetic and compiles it into a fixed circuit; with non-stationary encodings, it does not.

\paragraph{Falsification and grokking.}
The integer-token model also handles \emph{structure falsification}: when rotation-consistent evidence is followed by a contradicting observation, the model's entropy spikes in a single step (tracking the Bayesian posterior collapse from $P(\Rot) \approx 1$ to $P(\Rot) = 0$), with no perseveration.
Without curriculum learning, the model exhibits grokking-style dynamics~\citep{power2022grokking}: 80K steps of plateau (pure bijection elimination) followed by sudden emergence of model selection.
A curriculum annealing $\pi$ from 0.95 to 0.5 over 30K warmup steps accelerates convergence by $\sim$5$\times$.

\paragraph{Implication.}
The rotation experiments demonstrate that a small transformer can perform quantitatively correct Bayesian model selection, but only when the structural hypothesis is \emph{perceptually accessible} through the token encoding.
The stationarity control and header subtask diagnostics sharpen this: the operative factor is not integer identity but stationary semantics, and the failure localizes not to header parsing (which succeeds perfectly) but to the composition of header inversion with modular arithmetic.
This raises the central question of this paper: is the perceptual access condition an inherent limitation of gradient-compiled inference, or can model selection succeed with opaque symbols when the structural hypothesis is inherently label-invariant?



\section{Label-Invariant Model Selection}
\label{sec:variant3}

The perceptual access condition identified in \Cref{sec:rotation-experiments} motivates a new experimental design: can we find a structured class whose defining property is \emph{label-invariant}---detectable without arithmetic over token identities?
If so, model selection should succeed even with opaque symbols.

\subsection{Fixed-Point-Free Involutions}
\label{sec:involutions}

\begin{definition}[Fixed-Point-Free Involution Class]
\label{def:involution}
A \emph{fixed-point-free involution} on $\Zn$ (with $n$ even) is a permutation $f$ satisfying $f(f(x)) = x$ and $f(x) \neq x$ for all $x$.
Equivalently, $f$ is a product of $n/2$ disjoint transpositions---a perfect matching on $\Zn$.
We denote this class:
\begin{equation}
  \mathrm{Inv} = \{ f \in \Sn : f \circ f = \mathrm{id},\; f(x) \neq x \text{ for all } x \}.
\end{equation}
The size of this class is:
\begin{equation}
  |\mathrm{Inv}_n| = (n-1)!! = (n-1)(n-3)(n-5) \cdots 3 \cdot 1.
\end{equation}
For $n = 16$: $|\mathrm{Inv}_{16}| = 15!! = 2{,}027{,}025$, compared to $|\Bij_{16}| = 16! \approx 2 \times 10^{13}$.
\end{definition}

\begin{remark}[Label invariance]
\label{rem:label-invariance}
The involution property $f(f(x)) = x$ is purely combinatorial: it depends only on the \emph{graph structure} of $f$ (every connected component is a 2-cycle), not on any ordering or arithmetic over the symbol set.
If $\sigma$ is any relabeling (bijection on $\Zn$), then $f$ is an involution if and only if $\sigma \circ f \circ \sigma^{-1}$ is an involution.
This means the involution-vs-bijection wind tunnel can use opaque symbols with no header encoding, directly testing whether model selection depends on perceptual access to algebraic structure.
\end{remark}

\subsection{Posterior Computation}
\label{sec:inv-posterior}

The posterior computation requires tracking the set of constraints imposed by observations on the space of consistent involutions.

\begin{definition}[Observation State]
\label{def:obs-state}
After $k$ observations $\D_k = \{(x_1, y_1), \ldots, (x_k, y_k)\}$ with distinct inputs, define:
\begin{itemize}[itemsep=2pt]
  \item \textbf{Confirmed 2-cycles} $C$: pairs $\{x_i, y_i\}$ where both $(x_i, y_i)$ and $(y_i, x_i)$ have been observed---i.e., $y_i$ also appeared as an input with output $x_i$.
  \item \textbf{Pending constraints} $P$: pairs $(x_i, y_i)$ where $y_i$ has not yet appeared as an input. These constrain $f(y_i) = x_i$ for any consistent involution.
  \item \textbf{Free elements} $F = \Zn \setminus (\text{all elements appearing in } C \text{ or } P)$: elements with no constraints.
\end{itemize}
Let $c = |C|$ (number of confirmed 2-cycles), $p = |P|$ (number of pending constraints), and $r = |F|$ (number of free elements).
Note that $2c + 2p + r = n$.
\end{definition}

\begin{proposition}[Involution Likelihood]
\label{prop:inv-likelihood}
If $\D_k$ is consistent with an involution (no contradictions detected), the number of involutions consistent with $\D_k$ is:
\begin{equation}
\label{eq:inv-consistent}
  N_{\mathrm{Inv}}(\D_k) = (r - 1)!!
\end{equation}
where $r = n - 2c - 2p$ is the number of free elements.
The marginal likelihood under the involution class is:
\begin{equation}
  P(\D_k \mid \mathrm{Inv}) = \frac{(r-1)!!}{(n-1)!!}.
\end{equation}
\end{proposition}

\begin{proof}
Each confirmed 2-cycle is fully determined.
Each pending constraint fixes one additional mapping ($f(y_i) = x_i$), which pairs $y_i$ with $x_i$---another determined 2-cycle once confirmed.
The remaining $r$ free elements must form $r/2$ transpositions, which can be done in $(r-1)!!$ ways (choose a partner for the first element in $r-1$ ways, then for the next unpaired element in $r-3$ ways, etc.).
\end{proof}

\begin{remark}
One might object that the model is ``just counting constraints.''
But this is precisely what Bayesian model selection reduces to: computing the marginal likelihood ratio, which requires tracking the number of consistent hypotheses under each class as evidence accumulates.
The non-trivial aspect is that the model must maintain this count \emph{implicitly} through its predictive distribution, dynamically updating as each observation constrains the involution class differently from the bijection class.
\end{remark}

\begin{proposition}[Involution Class Posterior]
\label{prop:inv-class-posterior}
Let $\mathbb{1}_{\mathrm{Inv}}(\D_k) = 1$ if $\D_k$ is consistent with an involution (no fixed points observed, no contradictions with the reciprocal property) and 0 otherwise.
Then:
\begin{equation}
\label{eq:inv-class-posterior}
  P(\mathrm{Inv} \mid \D_k) = \frac{\pi \cdot \frac{(r-1)!!}{(n-1)!!} \cdot \mathbb{1}_{\mathrm{Inv}}(\D_k)}{\pi \cdot \frac{(r-1)!!}{(n-1)!!} \cdot \mathbb{1}_{\mathrm{Inv}}(\D_k) + (1-\pi) \cdot \frac{(n-k)!}{n!}}.
\end{equation}
\end{proposition}

\begin{remark}[Predictive distribution at arbitrary positions]
\label{rem:inv-predictive}
The full Bayesian predictive for $f(x_{k+1})$ at an arbitrary position depends on whether $x_{k+1}$ is a previous output (reciprocal) or a fresh element.
At a \textbf{reciprocal} position where $x_{k+1} = y_i$ for some observed pair, the involution class deterministically predicts $f(x_{k+1}) = x_i$, while the bijection class assigns $1/(n-k)$ to each unused output, yielding a clean mixture analogous to \Cref{prop:predictive}.
At a \textbf{fresh} position where $x_{k+1}$ has not appeared as any previous output, the involution class assigns $1/(r-1)$ to each of the $r-1$ free elements that could be paired with $x_{k+1}$, while the bijection class assigns $1/(n-k)$ to all unused outputs.
The resulting predictive has a more complex atom structure (the sets of free elements and unused outputs partially overlap), and the mixture inversion to extract a class posterior is not as clean.
We therefore report class posteriors only at reciprocal positions, where the extraction is exact.
\end{remark}

\subsection{Two Regimes of Evidence}
\label{sec:inv-evidence}

The involution BWT features two qualitatively different types of evidence:

\paragraph{New-element observations.}
When input $x_j$ is a ``fresh'' element (not a previous output), observing $f(x_j) = y_j$ provides weak evidence for the involution class.
The per-observation Bayes factor is:
\begin{equation}
  \frac{P(y_j \mid \mathrm{Inv}, \D_{j-1})}{P(y_j \mid \Bij, \D_{j-1})} = \frac{1/(r-1)}{1/(n-j+1)} = \frac{n - j + 1}{r - 1}
\end{equation}
which is modest (close to 1 when $r \approx n - j + 1$).

\paragraph{Reciprocal observations.}
When input $x_j$ equals a previous output $y_i$, the involution \emph{deterministically} predicts $f(x_j) = x_i$ (since $f(f(x_i)) = x_i$).
Under the bijection class, $f(x_j)$ is uniform over $n - j + 1$ unused outputs.
The per-observation Bayes factor is:
\begin{equation}
\label{eq:reciprocal-bf}
  \frac{P(f(x_j) = x_i \mid \mathrm{Inv}, \D_{j-1})}{P(f(x_j) = x_i \mid \Bij, \D_{j-1})} = n - j + 1.
\end{equation}
For $n = 16$ and $j = 5$, this gives a single-observation Bayes factor of 12---comparable to the rotation evidence from 2 observations.

\begin{example}[$n = 16$, mixed observation types]
\label{ex:inv-n16}
Consider a sequence where inputs are ordered so that some reciprocal tests occur:

\medskip
\begin{center}
\begin{tabular}{cclcc}
\toprule
$k$ & Input & Type & Cumulative $\BF$ & $P(\mathrm{Inv} \mid \D_k)$ \\
\midrule
1 & fresh & New element & 1.07 & 0.52 \\
2 & fresh & New element & 1.23 & 0.55 \\
3 & $y_1$ & Reciprocal ($\to x_1$) & 17.2 & 0.95 \\
4 & fresh & New element & 20.3 & 0.95 \\
5 & $y_2$ & Reciprocal ($\to x_2$) & 244 & 0.996 \\
\bottomrule
\end{tabular}
\end{center}

\medskip
\noindent
The reciprocal observations at $k = 3$ and $k = 5$ provide the dominant evidence, each contributing a Bayes factor of $\sim$10--13$\times$.
The new-element observations provide weak corroborating evidence.
\end{example}

\subsection{Falsification in the Involution BWT}
\label{sec:inv-falsification}

The involution hypothesis can be falsified in two ways:
\begin{enumerate}[itemsep=2pt]
  \item \textbf{Fixed point}: observing $f(x) = x$ for any $x$.
  \item \textbf{Reciprocal violation}: observing $f(y_i) \neq x_i$ where $(x_i, y_i)$ was previously observed.
\end{enumerate}
In both cases, the posterior collapses to $P(\mathrm{Inv} \mid \D_k) = 0$ in a single step, exactly as in rotation falsification.

\subsection{Why Label Invariance Matters}
\label{sec:inv-significance}

The involution BWT is designed to isolate model selection from token-space arithmetic:
\begin{itemize}[itemsep=2pt]
  \item Detecting rotations requires computing $y - x \bmod n$, which depends on the cyclic ordering of tokens.
  \item Detecting involutions requires checking whether $f(y_i) = x_i$---a purely relational test that is invariant under any relabeling of the symbol set.
\end{itemize}
If the transformer performs quantitatively correct Bayesian model selection in the involution BWT \emph{with opaque symbols}, this would demonstrate that the ``scientist'' capability is not an artifact of arithmetic accessibility but a genuine capacity for combinatorial structure detection.

\subsection{Experiments}
\label{sec:inv-experiments}

We train the same 2.8M-parameter architecture on the involution-vs-bijection task with $n = 16$ and $\pi = 0.5$.
Input ordering interleaves fresh inputs with reciprocal tests: after every 2 fresh inputs $(x_i, f(x_i))$, we present a previous output $y_j$ as input, testing whether $f(y_j) = x_j$.
This interleaving is part of the generative distribution, not a form of synthetic supervision: the Bayesian optimal predictor is computed under the same input ordering, and the model receives no signal beyond the standard cross-entropy loss at each output position.
The ordering simply ensures that the training distribution includes the reciprocal positions where model selection produces a measurable signal---without them, neither the Bayesian agent nor the transformer would have evidence to distinguish involutions from bijections.
The implicit class posterior is extracted only at reciprocal positions, where the involution hypothesis makes a deterministic prediction and the mixture inversion is clean.
All reported MAE values are averaged over positions and sequences (1000 evaluation sequences, all reciprocal positions per sequence).

\paragraph{Integer tokens.}
With integer tokens and 300K training steps (cosine schedule), the model achieves:
\begin{itemize}[itemsep=2pt]
  \item Entropy MAE: \textbf{0.011 bits} (vs.\ 0.12 for rotation model selection)
  \item Class posterior MAE: \textbf{0.004} (vs.\ 0.04 for rotations)
\end{itemize}
This is near-perfect Bayesian model selection---an order of magnitude more accurate than the rotation variant.

\paragraph{Smooth learning dynamics.}
Unlike the rotation BWT, which exhibited 80K steps of plateau followed by sudden grokking, the involution model learns smoothly from the start: MAE is 0.058 at step 10K and steadily improves to 0.011 by step 230K.
This qualitative difference in learning dynamics likely reflects the nature of the evidence.
Detecting rotations requires learning modular arithmetic ($y - x \equiv c \bmod n$), a global algebraic relationship that may require circuit consolidation before it ``clicks.''
Detecting involutions requires recognizing that a current input matches a previous output---a local relational pattern naturally suited to attention-based matching.

\begin{figure}[t]
\centering
\includegraphics[width=\columnwidth]{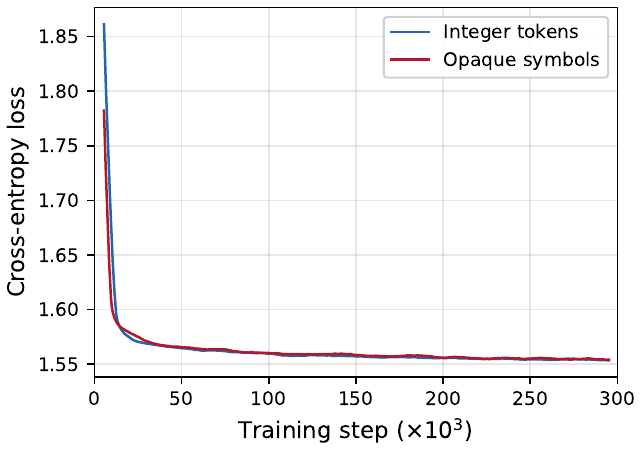}
\caption{Training loss curves for the involution BWT ($n = 16$, 300K steps).
Both integer-token and opaque-symbol models converge smoothly, with no grokking dynamics.
The opaque model converges slightly faster in early training due to the smaller effective vocabulary.}
\label{fig:training-loss}
\end{figure}

\paragraph{Opaque symbols.}
With opaque symbol encoding---where tokens are randomly relabeled per sequence with no header conveying structure---the model achieves MAE $= 0.009$ bits and class posterior MAE $= 0.003$ after 300K training steps.
This is the central result of the paper.
Recall that rotation model selection \emph{completely fails} with opaque symbols (MAE $= 1.28$ bits), because detecting $y = x + c \bmod n$ requires arithmetic over token identities.
Involution model selection succeeds because the defining property $f(f(x)) = x$ is purely relational: the model need only check whether a current input matches a previous output, a pattern that is invariant under any relabeling of the symbol set.
This demonstrates that Bayesian model selection is not an artifact of arithmetic accessibility---it is a genuine capacity for combinatorial structure detection that operates through attention-based relational matching.

\begin{figure}[t]
\centering
\includegraphics[width=\columnwidth]{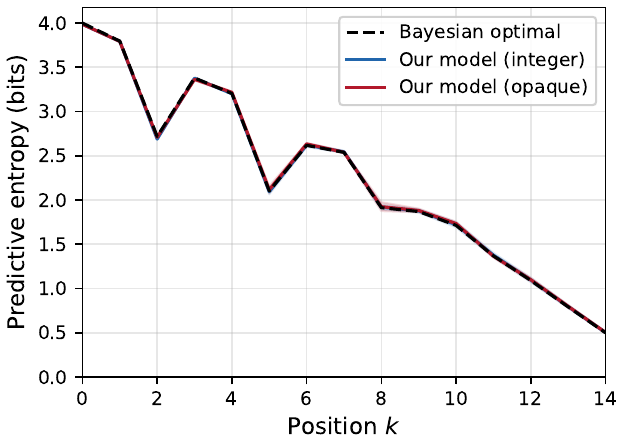}
\caption{Per-position predictive entropy for the involution BWT ($n = 16$, $\pi = 0.5$, mean $\pm$ std over 3 seeds).
Both integer-token and opaque-symbol models track the Bayesian optimal entropy (dashed) to within 0.01 bits across all 15 positions.
The characteristic sawtooth pattern reflects alternating fresh inputs (higher entropy) and reciprocal tests (lower entropy as the involution class concentrates probability on the deterministic prediction).}
\label{fig:entropy-tracking}
\end{figure}

\begin{figure}[t]
\centering
\includegraphics[width=\columnwidth]{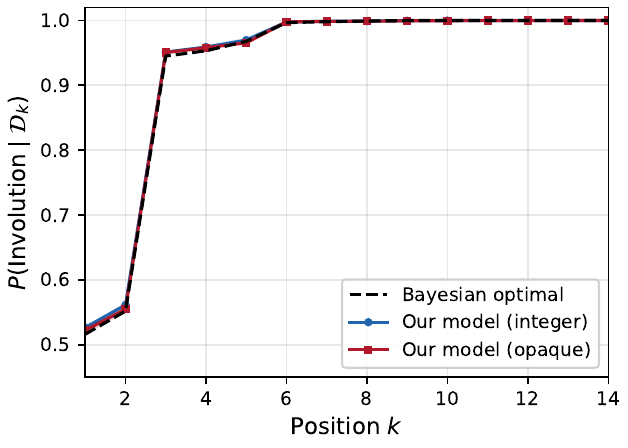}
\caption{Implicit class posterior $P(\mathrm{Inv} \mid \mathcal{D}_k)$ over positions (mean $\pm$ std over 3 seeds).
Both models implement calibrated Occam's razor: the posterior rises from $\approx 0.5$ at $k = 1$ toward 1.0 as reciprocal confirmations accumulate, tracking the Bayesian optimal posterior (dashed) with MAE $< 0.004$.}
\label{fig:class-posterior}
\end{figure}

\begin{figure}[t]
\centering
\includegraphics[width=\columnwidth]{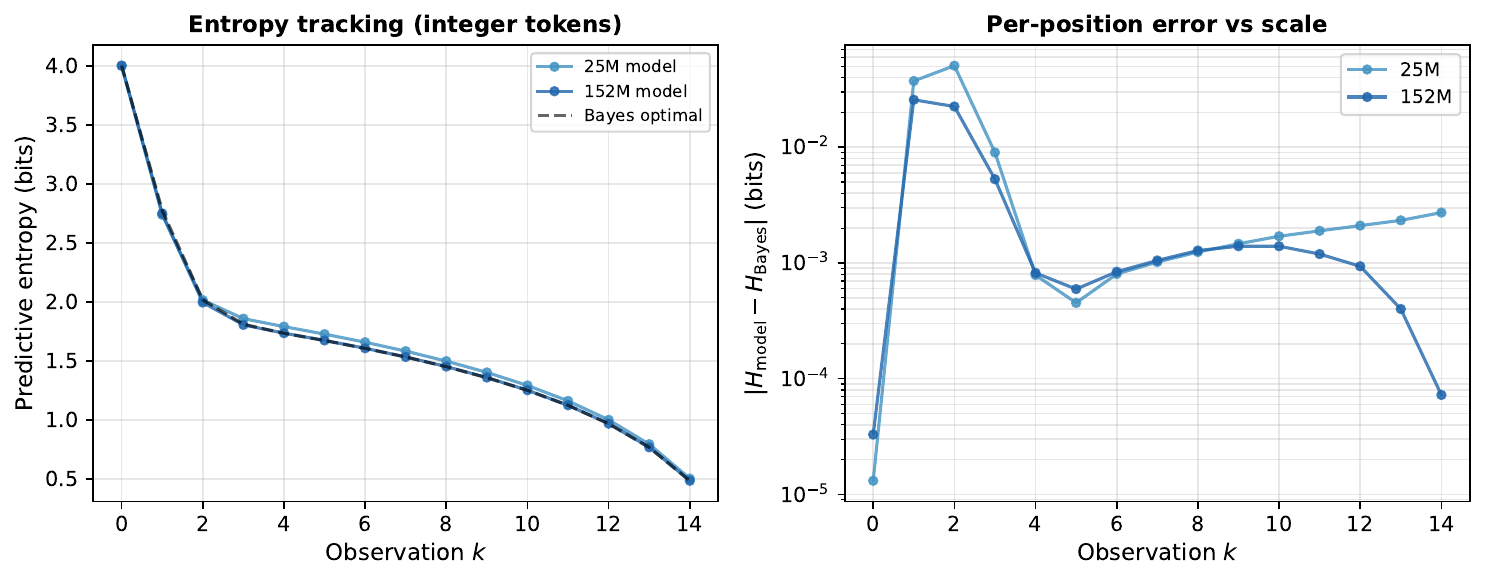}
\caption{Per-position entropy tracking for the rotation model-selection task with integer tokens at two scales.
\emph{Left}: both the 25M and 152M models overlay the Bayesian optimal entropy (dashed).
\emph{Right}: per-position absolute error.
The 152M model achieves errors below $10^{-4}$ bits at late positions, approaching the limits of floating-point arithmetic.}
\label{fig:scaling-precision}
\end{figure}

\section{Non-Nested Model Selection: Involutions vs.\ 3-Cycles}
\label{sec:nonnested}

All experiments so far compare a structured class to the full bijection class, a \emph{nested} comparison where the structured class is a strict subset of bijections.
In nested comparisons, model selection is dominated by a subset-elimination dynamic: the structured class is preferred unless the evidence rules it out, and the complexity penalty (the Occam factor) always favors the smaller class.
A stronger test asks whether transformers can perform model selection between two classes where \emph{neither is a subset of the other}, requiring genuine discrimination rather than simplicity or subset bias.
(The two classes are not equal in size---$|\mathrm{Inv}_{12}| = 10{,}395$ vs.\ $|\mathrm{Tri}_{12}| = 246{,}400$---so an Occam factor still operates; what the non-nested test removes is the subset-elimination shortcut available in every prior experiment.)

\subsection{3-Cycle Permutations}
\label{sec:three-cycles}

\begin{definition}[3-Cycle Class]
\label{def:three-cycle}
A \emph{fixed-point-free 3-cycle permutation} on $\Zn$ (with $n$ divisible by 3) is a permutation $f$ that is a product of $n/3$ disjoint 3-cycles.
Each element lies in exactly one 3-cycle $(a \to b \to c \to a)$, so $f(a) = b$, $f(b) = c$, $f(c) = a$.
We denote this class:
\begin{equation}
  \mathrm{Tri} = \{ f \in \Sn : f \text{ is a product of } n/3 \text{ disjoint 3-cycles} \}.
\end{equation}
The size of this class is:
\begin{equation}
\label{eq:tri-count}
  |\mathrm{Tri}_n| = \frac{n!}{3^{n/3} \cdot (n/3)!}.
\end{equation}
To derive this: partition $n$ elements into $n/3$ unordered triples in $n!/((3!)^{n/3} \cdot (n/3)!)$ ways, then orient each triple as one of $2$ non-identity cyclic permutations, giving:
\begin{equation}
  |\mathrm{Tri}_n| = \frac{n! \cdot 2^{n/3}}{(3!)^{n/3} \cdot (n/3)!} = \frac{n!}{3^{n/3} \cdot (n/3)!}.
\end{equation}
For $n = 12$: $|\mathrm{Tri}_{12}| = 12!/(3^4 \cdot 4!) = 246{,}400$.
\end{definition}

\begin{remark}[Non-nestedness]
\label{rem:nonnested}
$\mathrm{Inv}$ and $\mathrm{Tri}$ are disjoint: every involution consists entirely of 2-cycles, while every 3-cycle permutation consists entirely of 3-cycles.
No permutation belongs to both classes.
This is a qualitatively different test from all prior experiments, where the structured class was always a subset of $\Bij$.
\end{remark}

\begin{remark}[Label invariance of both classes]
\label{rem:tri-label-invariance}
Like the involution property, the 3-cycle property is purely combinatorial: if $\sigma$ is any relabeling, then $f$ is a product of disjoint 3-cycles if and only if $\sigma \circ f \circ \sigma^{-1}$ is.
Both classes are label-invariant, so the non-nested BWT can use opaque symbols for \emph{both} hypothesis classes.
\end{remark}

\subsection{One-Shot Discrimination}
\label{sec:nonnested-discrimination}

The key diagnostic signal arises at \emph{reciprocal positions}.
Given an earlier observation $(x_i, y_i)$, we present $y_i$ as input:
\begin{itemize}[itemsep=2pt]
  \item Under $\mathrm{Inv}$: $f(y_i) = x_i$ deterministically (since $f(f(x_i)) = x_i$).
  \item Under $\mathrm{Tri}$: $f(y_i) \neq x_i$.
    Instead, if $(x_i, y_i, z_i)$ is the 3-cycle, then $f(y_i) = z_i$---a \emph{third} element, distinct from both $x_i$ and $y_i$.
\end{itemize}
A single reciprocal observation therefore provides \emph{one-shot} class discrimination: if $f(y_i) = x_i$, the 3-cycle class is falsified at that position; if $f(y_i) \neq x_i$ (and $f(y_i) \neq y_i$), the involution class is falsified.
This clean separation at reciprocal positions is what makes the non-nested comparison tractable.

\subsection{Experiments}
\label{sec:nonnested-experiments}

We train the same 2.8M-parameter architecture on the involution-vs-3-cycle task with $n = 12$ (the smallest value divisible by both 2 and 3 that is large enough for meaningful model selection) and $\pi = 0.5$.
Input ordering interleaves fresh inputs with reciprocal tests, as in the involution BWT (\Cref{sec:inv-experiments}).
Each model is trained for 300K steps with cosine schedule, and we report results averaged over 3 seeds (42, 137, 2024).

\begin{table}[h]
\centering
\caption{Non-nested model selection: involutions vs.\ 3-cycles ($n = 12$, $\pi = 0.5$, mean $\pm$ std over 3 seeds).
Both integer-token and opaque-symbol models achieve near-exact Bayesian model selection, confirming that the perceptual access condition depends on the \emph{structure} being detected, not on whether classes are nested.}
\label{tab:nonnested}
\begin{tabular}{lcc}
\toprule
Token encoding & Entropy MAE (bits) & Class Post.\ MAE \\
\midrule
Integer tokens & $0.077 \pm 0.001$ & $0.001$ \\
Opaque symbols & $0.076 \pm 0.002$ & $0.001$ \\
\bottomrule
\end{tabular}
\end{table}

\paragraph{Results.}
\Cref{tab:nonnested} shows the results and \Cref{fig:nonnested-entropy,fig:nonnested-posterior} display the per-position tracking.
Both configurations achieve near-exact Bayesian model selection: entropy MAE of 0.077 bits (integer) and 0.076 bits (opaque), with class posterior MAE under 0.001 in both cases.
The model correctly identifies which class generated the data and adjusts its predictions accordingly, with performance essentially independent of token encoding.

\begin{figure}[t]
\centering
\includegraphics[width=\columnwidth]{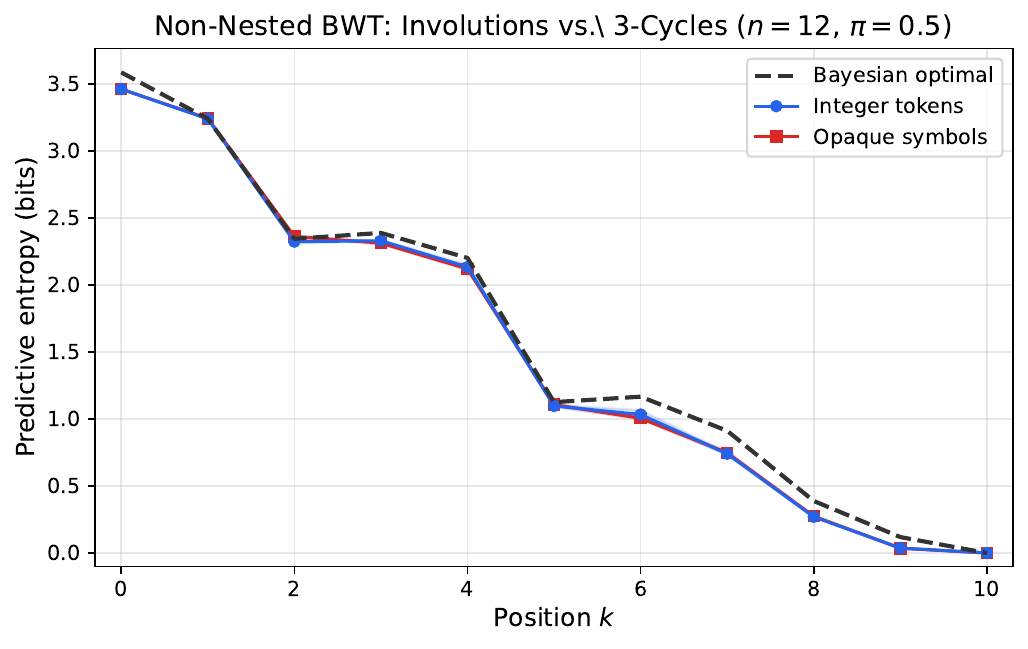}
\caption{Per-position predictive entropy for the non-nested BWT (involutions vs.\ 3-cycles, $n = 12$, $\pi = 0.5$, mean $\pm$ std over 3 seeds).
Both integer-token and opaque-symbol models track the Bayesian optimal entropy (dashed) across all positions.
The higher baseline entropy compared to the nested involution BWT (\Cref{fig:entropy-tracking}) reflects the harder prediction problem when both competing classes make structured predictions.}
\label{fig:nonnested-entropy}
\end{figure}

\begin{figure}[t]
\centering
\includegraphics[width=\columnwidth]{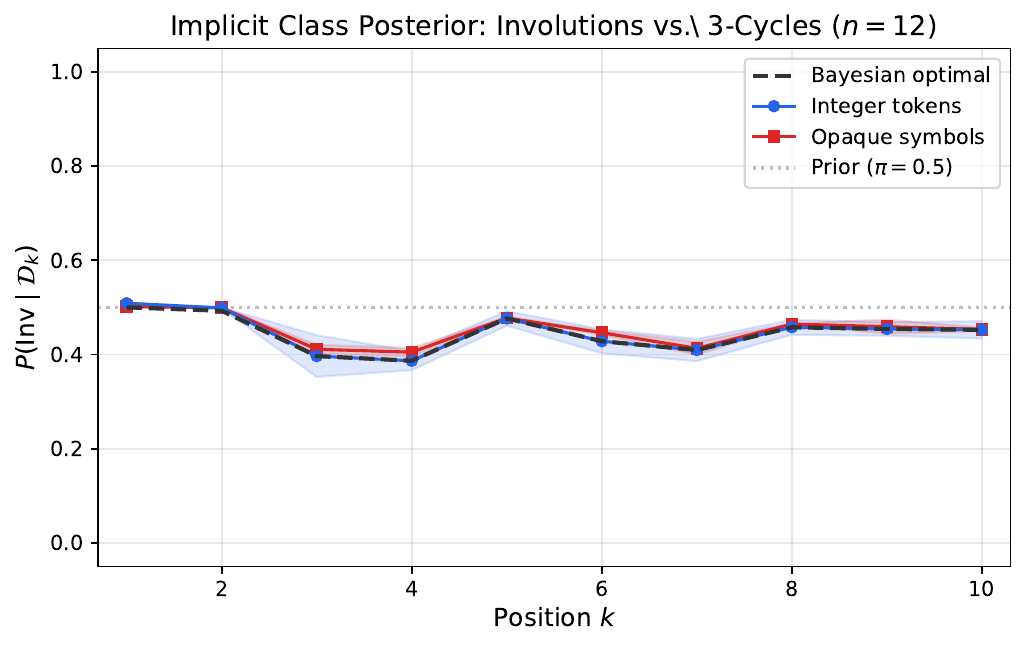}
\caption{Implicit class posterior $P(\mathrm{Inv} \mid \mathcal{D}_k)$ at reciprocal positions for the non-nested BWT (mean $\pm$ std over 3 seeds).
Both models track the Bayesian optimal posterior (dashed) with MAE $< 0.001$, exploiting the one-shot discrimination property: a single reciprocal observation deterministically distinguishes involutions from 3-cycles.}
\label{fig:nonnested-posterior}
\end{figure}

\paragraph{Comparison with nested experiments.}
The non-nested results complement the nested involution-vs-bijection results (\Cref{sec:inv-experiments}) in two ways.
First, the higher entropy MAE (0.077 vs.\ 0.011 bits) reflects the harder prediction problem: in the nested setting, the bijection class assigns $1/(n-k)$ probability to each unused output, while in the non-nested setting, both classes make structured predictions at reciprocal positions, requiring the model to maintain two competing structural hypotheses simultaneously.
Second, class posterior accuracy is \emph{better} in the non-nested case (MAE $= 0.001$ vs.\ 0.004), reflecting the one-shot discrimination property: a single reciprocal observation unambiguously distinguishes the two classes, whereas in the nested case, reciprocal evidence is strong but not deterministically discriminating.

\paragraph{Perceptual access confirmed.}
The key finding is that opaque symbols perform identically to integer tokens on the non-nested task.
This is expected---both involutions and 3-cycles are label-invariant---but it completes a clean $2 \times 2$ dissociation:

\begin{center}
\begin{tabular}{lcc}
\toprule
& Integer tokens & Opaque symbols \\
\midrule
Arithmetic: addition (rotations) & 0.12 bits & 1.28 bits (fails) \\
Arithmetic: multiplication ($f(x) = cx \bmod 17$) & 0.03 bits & 1.17 bits (fails) \\
Arithmetic: polynomial ($d = 1$--$5$, vs.\ bijection) & 1.05--1.42 bits (fails) & --- \\
Relational structure (involutions) & 0.011 bits & 0.009 bits \\
Relational (non-nested: inv.\ vs.\ 3-cyc.) & 0.077 bits & 0.076 bits \\
\bottomrule
\end{tabular}
\end{center}

\noindent
The perceptual access condition is sharp: opaque symbols destroy performance when the structure requires arithmetic---whether modular addition (rotations) or modular multiplication ($f(x) = cx \bmod p$)---but have no effect when the structure is purely relational (involutions, 3-cycles), regardless of whether the comparison is nested or non-nested.

\subsection{Probing Pretrained LLMs}
\label{sec:llm-probing}

Since the involution BWT uses integer tokens, we can probe frontier LLMs directly: present input-output pairs as text and extract the model's predictive distribution over the next output.
Each prompt states ``A function f maps integers to integers.\ The domain is \{0, 1, \ldots, 15\}.\ f is a bijection.'' followed by observed pairs (e.g., ``f(3) = 9'', ``f(7) = 2'', \ldots) and a query ``f(9) ='' at a reciprocal position.
For models exposing log-probabilities (OpenAI), we use the top-5 logprobs and distribute remaining probability mass uniformly over unobserved tokens; this is an approximation that may slightly distort the tails of the distribution.
For sampling-only APIs (Claude, Gemini), we draw 20 samples per position and build empirical distributions, which introduces sampling noise.
We evaluate on 100 sequences (50 for Claude due to cost), probing at positions $k \in \{0, 1, \ldots, 7\}$.
Given these methodological limitations, we emphasize qualitative trends over precise numeric rankings.
Bootstrap 95\% confidence intervals on entropy MAE (resampling over sequences) are $\pm 0.08$ bits for logprob-based models and $\pm 0.15$ bits for sampling-based models, confirming that the qualitative gap between frontier LLMs and our purpose-trained model is robust despite measurement noise.

\begin{table}[h]
\centering
\caption{Frontier LLM performance on the involution model-selection BWT ($n = 16$, $\pi = 0.5$).
All models show qualitative Bayesian behavior---entropy decreasing with evidence, class posteriors moving in the correct direction---but purpose-training closes a $55\times$ calibration gap.}
\label{tab:llm-probing}
\begin{tabular}{lcc}
\toprule
Model & Entropy MAE (bits) & Class Post.\ MAE \\
\midrule
Claude Sonnet 4.5 & 0.55 & 0.50 \\
Gemini 2.0 Flash & 0.95 & 0.36 \\
GPT-4o-mini & 1.07 & 0.42 \\
GPT-4.1 & 0.83 & 0.29 \\
GPT-4o & 0.60 & 0.18 \\
\midrule
Our model (integer) & \textbf{0.011} & \textbf{0.004} \\
Our model (opaque) & \textbf{0.009} & \textbf{0.003} \\
\bottomrule
\end{tabular}
\end{table}

\begin{figure}[t]
\centering
\includegraphics[width=\columnwidth]{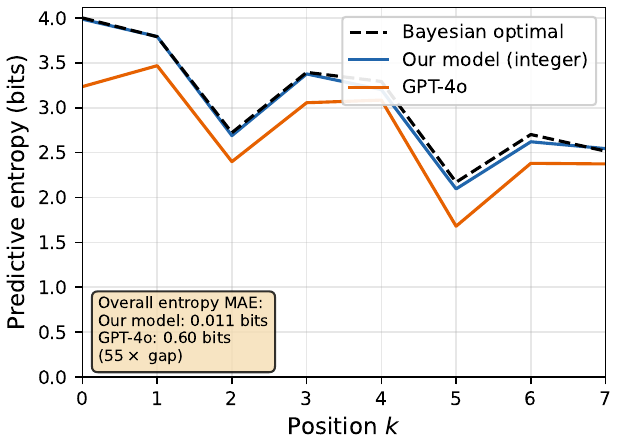}
\caption{Predictive entropy vs.\ position for the involution BWT ($n = 16$, $\pi = 0.5$).
GPT-4o tracks the qualitative shape of the Bayesian curve but with a $55\times$ calibration gap relative to our purpose-trained model, which overlaps the Bayesian optimum (dashed).}
\label{fig:llm-gap}
\end{figure}

Frontier LLMs show the qualitative signatures of Bayesian model selection: their entropy decreases with observations and their implicit class posteriors move in the correct direction.
This is notable given that these models were never trained on involution tasks.
However, they are not calibrated: the best (GPT-4o, 0.60-bit MAE) has a class posterior MAE of 0.18, compared to 0.004 for our purpose-trained model, and some models (e.g., Claude Sonnet 4.5) sit near a class-posterior MAE of 0.50---essentially no calibrated class selection under the $\pi = 0.5$ prior.
We stress that this is not an apples-to-apples comparison: the purpose-trained model is evaluated on its exact training distribution with a full-vocabulary softmax, while the frontier models are probed zero-shot through lossy logprob and sampling estimates, so the gap conflates distribution match with measurement method.
The directional qualitative-to-quantitative gap is nonetheless consistent across model families (OpenAI, Anthropic, Google) and scales---consistent with, though not proof of, an architectural inductive bias toward Bayesian reasoning that purpose-training sharpens.

\subsection{Implicit Knowledge vs.\ Explicit Reasoning}
\label{sec:cot-dissociation}

A natural question is whether frontier LLMs can improve their model-selection performance through chain-of-thought reasoning.
We probe GPT-4o, Claude Sonnet 4.5, and Gemini 2.0 Flash in a \emph{reasoning} mode: each prompt asks the model to ``think step by step'' before answering, and we extract the final integer prediction from the response text.

The results reveal a striking dissociation.
Reasoning-mode accuracy is uniformly poor: GPT-4o achieves 9.0\%, Claude 6.0\%, and Gemini 7.5\%---all near the random baseline of $1/16 = 6.25\%$.
In contrast, GPT-4o's \emph{implicit} next-token logprobs achieve 21.4\% accuracy and place 98.7\% probability on the correct answer at strong reciprocal positions.

This is a comparison between two different readouts of the same model---the verbalized final answer under chain-of-thought versus the implicit next-token distribution---so it does not isolate a causal effect of reasoning on the underlying representation.
What it does show is a dissociation: when prompted to reason explicitly, GPT-4o articulates the bijection constraint correctly (``since $f$ is a bijection, the output must be one of the remaining values\ldots'') but does not identify the involution pattern, ultimately guessing among consistent outputs, whereas its implicit next-token distribution concentrates mass on the reciprocal prediction $f(y_i) = x_i$.

This dissociation between implicit distributional behavior and explicit verbal reasoning is suggestive rather than conclusive: it rests on a single prompt template and the small samples of \Cref{sec:llm-probing}.
Taken at face value, it indicates that the Bayesian signal present in the model's implicit next-token distribution is not surfaced by its verbalized reasoning on this task---a pattern worth probing more systematically in future work.

\section{The Polynomial Barrier: Circuit Depth as a Learnability Boundary}
\label{sec:polynomial}

The perceptual access condition identifies \emph{whether} arithmetic is required.
We now ask a finer question: \emph{how much} arithmetic?
Polynomial detection provides a natural hierarchy of circuit complexity: a degree-$d$ polynomial $f(x) = a_d x^d + \cdots + a_1 x + a_0 \bmod p$ is determined by $d+1$ coefficients, so the first $d+1$ observations are near-uninformative (the Bayes factor stays within a factor $(p-1)/p$ of unity) and every subsequent observation multiplies the Bayes factor by $\sim p$.
Critically, \emph{verifying} that $k > d+1$ observations lie on a degree-$d$ polynomial requires solving a $(d+1) \times (d+1)$ Vandermonde system over $\mathbb{Z}_p$---Lagrange interpolation modulo $p$---a circuit whose depth grows with $d$.

\subsection{Setup}

Fix a prime $p$ and degree $d \geq 1$.
The structured class is:
\begin{equation}
  \mathcal{P}_d = \{ x \mapsto a_d x^d + \cdots + a_0 \bmod p \mid a_0, \ldots, a_d \in \mathbb{Z}_p,\; a_d \neq 0 \},
  \quad |\mathcal{P}_d| = (p-1) \cdot p^d.
\end{equation}
The alternative class is the full set of bijections $\Bij$ on $\mathbb{Z}_p$.
(Note: not all degree-$d$ polynomials are bijections, but the Bayesian calculation conditions on the observed data being consistent with either class.)

\paragraph{Bayesian posterior.}
After $k$ observations $\D_k = \{(x_i, y_i)\}$ with distinct inputs:
\begin{itemize}[itemsep=2pt]
  \item If $k \leq d+1$: the data is consistent with some degree-$d$ polynomial for \emph{any} set of distinct-input, distinct-output pairs (by Lagrange interpolation, $d+1$ points determine a unique polynomial, and any $k < d+1$ points are consistent with $p^{d-k}(p-1)$ polynomials of exactly degree $d$, i.e. with $a_d \neq 0$).
  The Bayes factor is:
  \begin{equation}
    \BF(\mathcal{P}_d : \Bij \mid \D_k) = \frac{p^{d-k}(p-1)}{(p-1)p^d} \cdot \frac{p!}{(p-k)!} = \frac{p!}{p^{k}(p-k)!},
  \end{equation}
  which equals $\prod_{i=0}^{k-1}(p-i)/p$.
  For $k = 1$ this is 1; for $2 \le k \le d+1$ it is slightly below 1 (e.g., $(p-1)/p$ at $k=2$), confirming that the first $d+1$ observations are near-uninformative.

  \item If $k > d+1$: the $d+1$ coefficients are uniquely determined by the first $d+1$ points.
  Each subsequent observation is a deterministic check: $P(y_j \mid \D_{j-1}, \mathcal{P}_d) = 1$ if the point lies on the polynomial, $0$ otherwise.
  Meanwhile, $P(y_j \mid \D_{j-1}, \Bij) = 1/(p-j+1)$.
  Once the data is consistent with a polynomial through $d+1$ points, each additional observation multiplies the Bayes factor by $p - j + 1$, yielding super-exponential growth.
\end{itemize}

\paragraph{Why this probes circuit depth.}
Detecting whether $k > d+1$ points lie on a degree-$d$ polynomial requires recovering the coefficients by solving:
\begin{equation}
  \begin{pmatrix} 1 & x_1 & \cdots & x_1^d \\ \vdots & \vdots & \ddots & \vdots \\ 1 & x_{d+1} & \cdots & x_{d+1}^d \end{pmatrix}
  \begin{pmatrix} a_0 \\ \vdots \\ a_d \end{pmatrix}
  =
  \begin{pmatrix} y_1 \\ \vdots \\ y_{d+1} \end{pmatrix}
  \pmod{p}
\end{equation}
and then checking $a_d x^d + \cdots + a_0 \equiv y_j \pmod{p}$ for $j > d+1$.
For $d = 1$ (affine maps), this reduces to modular subtraction and division---comparable to rotation detection.
For $d \geq 2$, the Vandermonde system requires modular polynomial arithmetic of increasing depth.
This provides a controlled knob on the computational complexity of the discriminative statistic.

\subsection{Experiments}

We test degree $d \in \{1, 2, 3, 4, 5\}$ on $\mathbb{Z}_{17}$ with $\pi = 0.5$, using the same architecture as previous experiments (6 layers, 6 heads, ${\sim}$2.8M parameters).
Each run trains for 150{,}000 steps with 3 random seeds.
We test polynomial-vs-bijection, where the alternative class provides the same bijection elimination bootstrap signal that enables rotation model selection.

\paragraph{Results.}
No degree shows any learning.

\begin{table}[h]
\centering
\caption{Polynomial degree-$d$ vs.\ bijection on $\mathbb{Z}_{17}$ (150K steps, 3 seeds, mean $\pm$ std).
All degrees fail regardless of bijection elimination signal.}
\label{tab:polynomial}
\begin{tabular}{lccc}
\toprule
Degree $d$ & Entropy MAE (bits) & KL (nats) & Class Post.\ MAE \\
\midrule
1 (affine) & $1.11 \pm 0.02$ & $0.76 \pm 0.01$ & $0.87 \pm 0.01$ \\
2 (quadratic) & $1.06 \pm 0.01$ & $0.73 \pm 0.01$ & $0.88 \pm 0.00$ \\
3 (cubic) & $1.05 \pm 0.03$ & $0.72 \pm 0.03$ & $0.88 \pm 0.01$ \\
4 (quartic) & $1.42 \pm 0.04$ & $0.99 \pm 0.03$ & $0.91 \pm 0.00$ \\
5 (quintic) & $1.27 \pm 0.04$ & $0.88 \pm 0.03$ & $0.90 \pm 0.01$ \\
\midrule
Rotation ($d = 0$, addition) & $\mathbf{0.12}$ & --- & $\mathbf{0.04}$ \\
Scalar mult.\ ($cx \bmod 17$) & $\mathbf{0.03}$ & --- & --- \\
\bottomrule
\end{tabular}
\end{table}

Every degree produces entropy MAE above 1~bit and class posterior MAE above 0.85---comparable to the opaque rotation failure ($1.28$~bits).
The bijection alternative provides the same elimination bootstrap that enables rotation model selection ($0.12$~bits) and scalar multiplication ($0.03$~bits), yet the model learns nothing about polynomial structure.

\paragraph{The polynomial learnability boundary.}
The failure is uniform across degrees, including $d = 1$, which is the affine class $f(x) = ax + b \bmod p$.
This is surprising: affine detection requires only one more arithmetic operation than rotation detection ($y = ax + b$ vs.\ $y = x + c$), yet the model goes from $0.12$-bit success to $1.11$-bit failure.

The explanation is that rotation detection has a uniquely shallow circuit.
Given one observation $(x_1, y_1)$, the rotation shift is $c = y_1 - x_1 \bmod n$, and verification at any subsequent point requires only $y_j \stackrel{?}{=} x_j + c \bmod n$---a single modular addition.
Affine detection, by contrast, requires two observations to determine two parameters ($a$ and $b$), involving modular division: $a = (y_2 - y_1)(x_2 - x_1)^{-1} \bmod p$.
Modular inversion is a qualitatively harder circuit; both integer and opaque tokens fail on affine detection at $p = 17$ under our standard (random-input) protocol, consistent with the barrier being computational rather than representational.

\paragraph{Scaling stress test.}
To rule out a capacity explanation, we run degree-1 (affine) polynomial detection on a 25M-parameter model (8 layers, 8 heads, $d_{\mathrm{model}} = 512$, $d_{\mathrm{ff}} = 2048$)---the same architecture that achieves $0.008$-bit MAE on rotation detection.
Across 8 random seeds, the 25M model produces $1.12 \pm 0.02$-bit entropy MAE and $0.87 \pm 0.01$ class posterior MAE---identical to the 2.8M model (Table~\ref{tab:polynomial-scaling}).
A $9\times$ increase in parameters moves the needle by exactly zero.

\begin{table}[h]
\centering
\caption{Scaling comparison: rotation vs.\ affine polynomial detection on $\mathbb{Z}_{17}$.
The 25M model that sharpens rotation detection by $15\times$ shows no improvement on affine detection.}
\label{tab:polynomial-scaling}
\begin{tabular}{llcc}
\toprule
Task & Model & Entropy MAE (bits) & Class Post.\ MAE \\
\midrule
Rotation ($y - x \bmod n$) & 2.8M & 0.12 & 0.04 \\
Rotation ($y - x \bmod n$) & 25M & \textbf{0.008} & --- \\
\midrule
Affine ($y = ax + b \bmod p$) & 2.8M (3 seeds) & $1.11 \pm 0.02$ & $0.87 \pm 0.01$ \\
Affine ($y = ax + b \bmod p$) & 25M (8 seeds) & $1.12 \pm 0.02$ & $0.87 \pm 0.01$ \\
\bottomrule
\end{tabular}
\end{table}

This indicates the polynomial barrier is not a capacity limitation: the same scaling that sharpens rotation detection by $15\times$ (from $0.12$ to $0.008$ bits) does not begin to crack affine detection.
We frame this as a learnability boundary rather than a formal circuit-complexity separation.
We establish it empirically and rule out capacity, but we do not exhibit a circuit-depth lower bound, and the optimization interventions that surfaced rotation model selection (curriculum annealing, extended training) were not applied to the affine case.
Whether the polynomial circuits are unlearnable in principle or merely beyond the reach of the optimizers we tried is left open; connecting the boundary to formal expressivity results for bounded-depth transformers~\citep{merrill2023parallelism,liu2023shortcuts,weiss2021rasp} is a natural next step.

\paragraph{Comparison to the perceptual access dissociation.}
The polynomial result reveals a finer structure within the ``arithmetic requires integer tokens'' side of the perceptual access condition.
Even with integer tokens---where the model has direct access to numerical values---and even at 25M parameters, the Vandermonde circuit for polynomial detection exceeds what gradient descent compiles.
Rotation detection succeeds because it requires a \emph{depth-1} arithmetic circuit (one subtraction).
Everything beyond that---affine maps ($d = 1$), quadratic polynomials ($d = 2$), and higher---fails uniformly.

\begin{center}
\begin{tabular}{lcc}
\toprule
Circuit type & MAE (bits) & Outcome \\
\midrule
Relational (no arithmetic) & 0.009--0.077 & Exact inference \\
Depth-1 arithmetic (rotation, $y - x \bmod n$) & 0.03--0.12 & Exact inference \\
Depth-2+ arithmetic (affine, polynomial) & 1.05--1.42 & Complete failure \\
\quad \textit{including 25M-param affine} & \textit{1.12} & \textit{Complete failure} \\
\bottomrule
\end{tabular}
\end{center}

\noindent
The perceptual access condition thus has two thresholds: (i)~relational vs.\ arithmetic, governing whether opaque symbols succeed, and (ii)~shallow vs.\ deep arithmetic, governing whether even integer tokens suffice.
The boundary of gradient-compiled inference is not simply ``can the model do arithmetic?'' but ``can the model compile the specific arithmetic circuit required for this discriminative statistic?''
This second threshold does not yield to scaling.

\section{Discussion}
\label{sec:discussion}

The model-selection BWT bridges the simulator and scientist views of in-context learning introduced in \Cref{sec:intro}.
Our experiments demonstrate that these are two manifestations of a single Bayesian computation: the same transformer that performs bijection elimination also performs quantitatively correct model selection---both in nested comparisons (involution-vs-bijection, 0.01-bit entropy agreement with the Bayesian optimum, even with opaque symbols) and in non-nested comparisons (involution-vs-3-cycle, 0.077-bit MAE with class posterior MAE under 0.001).

\paragraph{From rotations to involutions.}
The rotation experiments reveal a perceptual access condition: model selection succeeds with integer tokens (MAE $= 0.12$ bits) but fails with opaque symbols (MAE $= 1.28$ bits), because detecting rotations requires modular arithmetic.
The involution experiments resolve this by using a label-invariant structured class.
With integer tokens, the involution model achieves an order-of-magnitude improvement (MAE $= 0.011$ bits, class posterior MAE $= 0.004$), and learns smoothly without the grokking dynamics that characterize rotation model selection.
The qualitative difference---smooth vs.\ grokking---likely reflects the computational demands: involution detection requires only relational matching (``is this input a previous output?''), while rotation detection requires learning a global algebraic relationship.

\paragraph{Multiplication confirms the arithmetic boundary.}
The perceptual access condition is not specific to modular addition.
We test scalar multiplication maps $f(x) = cx \bmod p$ on the multiplicative group $\mathbb{Z}_p^* = \{1, \ldots, p-1\}$, which have the same Bayes factor growth as rotations (both are determined after one observation, with $|\text{Class}| = p - 1$) but require modular multiplication rather than addition.
With integer tokens, the model achieves $0.03$-bit entropy MAE at $p = 17$ and $p = 13$---near-Bayesian optimal.
With opaque symbols, the model fails completely: $1.17$-bit MAE at $p = 17$, indistinguishable from the rotation boundary.
The dissociation is $39\times$, confirming that the perceptual access condition applies to any arithmetic operation, not only to addition.

\paragraph{The polynomial barrier reveals circuit-depth dependence.}
The polynomial experiments (\Cref{sec:polynomial}) reveal a finer structure: even with integer tokens and the same bijection alternative, polynomial degree detection fails uniformly for $d \geq 1$.
This establishes that the boundary of gradient-compiled inference depends not just on \emph{whether} arithmetic is required, but on the \emph{depth} of the required circuit.
Rotation detection ($y - x \bmod n$) is uniquely shallow---a single subtraction---while polynomial detection requires Lagrange interpolation, a qualitatively deeper computation.
The failure of $d = 1$ (affine) is particularly revealing: one additional arithmetic operation beyond rotation detection---modular division to recover the slope---is sufficient to cross the circuit-complexity wall.
Crucially, this boundary does not yield to scaling: a 25M-parameter model that achieves $0.008$-bit rotation MAE produces identical $1.12$-bit affine MAE as the 2.8M model, consistent with a depth-dependent learnability boundary rather than a capacity limitation.

\paragraph{Grokking as circuit consolidation.}
The delayed emergence of rotation model selection---80K steps of plateau followed by a sudden transition---suggests that model selection over arithmetic structures requires a qualitatively different circuit from bijection elimination.
The involution model, by contrast, can build model-selection capability incrementally on top of standard attention-based matching, explaining the smooth learning curve.

\paragraph{Why attention naturally supports involution detection.}
The smooth learning dynamics of the involution model can be understood through the geometry of attention.
A standard attention head computes $\mathrm{softmax}(QK^\top/\sqrt{d})V$: the query--key dot product selects which previous tokens to attend to, and the value projection determines what information to extract.
Involution detection requires exactly this primitive: given a current input $x_j$, the model must check whether $x_j$ appeared as a \emph{previous output} $y_i$, and if so, retrieve the corresponding input $x_i$ to predict $f(x_j) = x_i$.
This is a key-matching-plus-value-retrieval computation that maps directly onto a single attention head---the induction-head primitive characterized by \citet{olsson2022induction}. We note this as an interpretive hypothesis: we do not probe the trained involution model for such a head directly, and mechanistic verification (head ablation, circuit identification) is left to future work.
By contrast, rotation detection requires computing $y - x \bmod n$---a global algebraic relationship that has no natural single-head implementation and instead requires coordinated multi-layer circuits, explaining the grokking dynamics.
This connection to attention geometry also explains why the involution BWT succeeds with opaque symbols: the key-matching operation $x_j \stackrel{?}{=} y_i$ depends only on token identity, not on any ordering or arithmetic over the token space.

\paragraph{Frontier LLMs as qualitative Bayesian model selectors.}
Probing five frontier LLMs on the involution task (\Cref{tab:llm-probing}) reveals that they exhibit the qualitative signatures of Bayesian model selection---entropy decreases with observations, and implicit class posteriors move in the correct direction---on a combinatorial task they were never explicitly trained on.
This is consistent with the thesis that the transformer architecture has an inductive bias toward Bayesian inference~\citep{aggarwal2025bayesian1}.
However, frontier models are not \emph{calibrated}: the best (GPT-4o, 0.60-bit MAE) is $55\times$ worse than our purpose-trained 2.8M-parameter model.
Strikingly, this Bayesian signal resides in the models' implicit next-token distributions rather than in their explicit reasoning: chain-of-thought prompting degrades accuracy to near-random levels, even as the models correctly articulate the bijection constraint.
Purpose-training on the correct data distribution sharpens the architecture's inherent Bayesian capacity from qualitative pattern to exact inference.

\paragraph{The dual-entropy ratio as boundary signature.}
Paper~I introduces a dual-entropy framework: context surprisal $H_I$ (the distinctiveness of context) and prediction entropy $H_P$ (the model's output uncertainty), with $\rho = H_P / H_I$ as a confidence-per-information coefficient~\citep{aggarwal2025bayesian1}.
In the wind tunnel, the Bayesian posterior entropy plays the role of $H_I$, and the entropy MAE we report is $|H_P - H_{\mathrm{Bayes}}|$---the gap between the model's confidence and the confidence warranted by the evidence.

The perceptual access condition maps directly onto this framework.
When the condition is met (integer-token rotations, scalar multiplication, involutions), $\rho$ decreases steadily during training: the model's predictive entropy tracks the Bayesian posterior more precisely with each checkpoint, reaching $0.009$--$0.12$ bits.
When the condition is violated (opaque-token rotations, affine and higher-degree polynomials), $\rho$ remains high throughout training: entropy MAE plateaus at $1.05$--$1.28$ bits from the earliest evaluation onward.

Paper~II provides the mechanistic explanation~\citep{aggarwal2025gradient2}: the advantage signal that drives routing crystallization during gradient descent is large when $H_I$ is high and $H_P$ is low.
Under non-stationary encodings or episode-dependent routing, the advantage signal averages to noise across episodes, $\rho$ fails to decrease, and no circuit compiles.
The dual-entropy ratio thus provides a unified quantitative signature of the learnability boundary: $\rho \to 0$ when and only when the discriminative circuit has fixed wiring across episodes.

\paragraph{Implications.}
These results extend the Bayesian wind tunnel methodology from posterior tracking within a class to posterior tracking \emph{over} classes.
A 2.8M-parameter transformer can perform exact Bayesian model selection---implementing Occam's razor with bit-level precision---when the structural hypothesis is detectable through the model's representational primitives.
The involution result---which succeeds even with opaque symbols---demonstrates that this capability does not depend on arithmetic; it requires only that the structured class be accessible through the model's attention mechanism.
The non-nested result (\Cref{sec:nonnested}) extends this further: model selection is not limited to simplicity-based Occam dynamics (nested classes) but generalizes to genuine discrimination between disjoint hypothesis classes with comparable complexity.
To our knowledge, this provides the first controlled, closed-form model-selection wind tunnel demonstrating that in-context learning implements the full Bayesian inference pipeline: model selection---including non-nested selection---followed by within-model filtering, with calibrated uncertainty at every stage.

\paragraph{Limitations and future work.}
Several directions remain open.
Our nested experiments use $n = 16$ (with the opaque rotation failure confirmed across $n \in \{8, 16, 32, 64\}$ and extended training to $500{,}000$ steps ruling out late grokking) and the non-nested experiments use $n = 12$ (the smallest value divisible by both 2 and 3), both with equal priors ($\pi = 0.5$); testing extreme priors would further probe the limits of model-selection precision.
The designed input ordering (interleaving fresh and reciprocal positions) ensures diagnostic evidence appears at predictable positions; randomizing the interrogation schedule would rule out any positional heuristic and strengthen the generality claim.
We extract class posteriors only at reciprocal positions where the mixture inversion is clean (\Cref{rem:inv-predictive}); extending the extraction to arbitrary positions would provide a denser evaluation signal.
Architecture ablations (removing specific heads or layers, comparing to non-attention baselines) and mechanistic probing (e.g., linear probes for class identity, or probes verifying that the model maintains representations of both hypothesis classes in the non-nested setting rather than simply recognizing 2-cycles) would provide direct evidence for the computational mechanisms underlying model selection.
Investigating whether a light auxiliary loss (e.g., predicting integer identities from opaque tokens using the header) during standard cross-entropy training enables subsequent success on opaque rotations---without architectural changes---would further characterize the boundary between the optimization objective and the architecture.
Finally, the frontier LLM evaluation relies on top-5 logprobs and low-sample Monte Carlo, which can bias entropy estimates; obtaining full $n$-way logprob distributions would yield cleaner calibration comparisons.
We plan to release the wind tunnel code and evaluation harnesses to facilitate independent replication.

\bibliographystyle{ACM-Reference-Format}
\bibliography{references,extra_refs}

\end{document}